\documentclass[12pt,twoside]{article}
\usepackage{latexsym,amsmath,amssymb}
\newdimen\paravsp  \paravsp=1.3ex
\def\paradot#1{\vspace{\paravsp plus 0.5\paravsp minus 0.5\paravsp}\noindent{\bf\boldmath{#1.}}}

\newtheorem{theorem}{Theorem}
\newtheorem{corollary}[theorem]{Corollary}

\newtheorem{definition}[theorem]{Definition}
\newtheorem{example}[theorem]{Example}
\newtheorem{proposition}[theorem]{Proposition}
\newtheorem{remark}[theorem]{Remark}
\newenvironment{proof}{{\noindent\bf Proof.}}{\vskip 1ex}

\newenvironment{keywords}{\centerline{\bf\small
Keywords}\begin{quote}\small}{\par\end{quote}\vskip 1ex}

\begin{document}

\title{
\vskip 2mm\bf\Large\hrule height5pt \vskip 4mm
Axioms for Rational Reinforcement Learning
\vskip 4mm \hrule height2pt}

\author{{\bf Peter Sunehag} and {\bf Marcus Hutter}\\[3mm]
\normalsize Research School of Computer Science\\[-0.5ex]
\normalsize Australian National University\\[-0.5ex]
\normalsize Canberra, ACT, 0200, Australia\\
\texttt{\small \{Peter.Sunehag,Marcus.Hutter\}@anu.edu.au}
}
\date{July 2011}

\maketitle

\begin{abstract}
We provide a formal, simple and intuitive theory of rational
decision making including sequential decisions that affect the
environment. The theory has a geometric flavor, which makes the
arguments easy to visualize and understand. Our theory is for
complete decision makers, which means that they have a complete set
of preferences. Our main result shows that a complete rational
decision maker implicitly has a probabilistic model of the
environment. We have a countable version of this result that brings
light on the issue of countable vs finite additivity by showing how
it depends on the geometry of the space which we have preferences
over. This is achieved through fruitfully connecting rationality
with the Hahn-Banach Theorem. The theory presented here can be
viewed as a  formalization and extension of the betting odds
approach to probability of Ramsey and De Finetti \cite{Ram31,deF37}.
\def\contentsname{\centering\normalsize Contents}\setcounter{tocdepth}{1}
{\parskip=-2.7ex\tableofcontents}
\end{abstract}


\newpage
\section{Introduction}

We study complete decision makers that can take a sequence of
actions to rationally pursue any given task. We suppose that the
task is described in a reinforcement learning framework where
the agent takes actions and receives observations and rewards. The
aim is to maximize total reward in some given sense.

Rationality is meant in the sense of internal consistency
\cite{Sugden}, which is how it has been used in \cite{NeuMor44} and
\cite{Sav54}. In \cite{NeuMor44}, it is proven that preferences
together with rationality axioms and probabilities for possible
events imply the existence of utility values for those events that
explain the preferences as arising through maximizing expected
utility. Their rationality axioms are
\begin{enumerate}
\item Completeness: Given any two choices we either prefer one of them
to the other or we consider them to be equally preferable;
\item Transitivity: A preferable to B and B to C imply A preferable to
C;
\item Independence: If A is preferable to B
and $t\in [0,1]$ then $tA+(1-t)C$ is preferable (or equal) to
$tB+(1-t)C$; \item Continuity: If A is preferable to B and B to C
then there exists $t\in [0,1]$ such that $B$ is equally preferable
to $tA+(1-t)C$.
\end{enumerate}
 In \cite{Sav54} the probabilities are not given but it
is instead proven that preferences together with rationality axioms
imply the existence of probabilities and utilities. We are here
interested in the case where one is given utility (rewards) and
preferences over actions and then deriving the existence of a
probabilistic world model. We put an emphasis on extensions to
sequential decision making with respect to a countable class of
environments. We set up simple axioms for a rational decision maker,
which implies that the decisions can be explained (or defined) from
probabilistic beliefs.

The theory of \cite{Sav54} is called subjective expected utility
theory (SEUT) and was intended to provide statistics with a strictly
behaviorial foundation. The behavioral approach stands in stark
contrast to approaches that directly postulate axioms that
``degrees of belief" should satisfy \cite{Cox46,Hal99,Jay03}. Cox's
approach \cite{Cox46,Jay03} has also been found \cite{Paris94} to
need additional technical assumptions in addition to the common
sense axioms originally listed by Cox. The original proof by
\cite{Cox46} has been exposed as not mathematically rigorous and his
theorem as wrong \cite{Hal99}. An alternative approach by
\cite{Ram31,deF37} is interpreting probabilities as fair betting
odds.

The theory of \cite{Sav54} has greatly influenced economics
\cite{Sugden} where it has been used as a description of rational
agents. Seemingly strange behavior was explained as having beliefs
(probabilities) and tastes (utilities) that were different from
those of the person to whom it looked irrational. This has turned
out to be insufficient as a description of human behavior
\cite{Allais,Ellsberg} and it is better suited as a normative theory
or design principle in artificial intelligence. In this article, we
are interested in studying the necessity for rational agents
(biological or not) to have a probabilistic model of their
environment. To achieve this, and to have as simple common sense
axioms of rationality as possible, we postulate that given any set
of values (a contract) associated with the possible events, the
decision maker needs to have an opinion on wether he prefers these
values to a guaranteed zero outcome or not (or equal). From this
setting and our other rationality axioms we deduce the existence of
probabilities that explain all preferences as maximizing expected
value. There is an intuitive similarity to the idea of
explaining/deriving probabilities as a bookmaker's betting odds as
done in \cite{deF37} and \cite{Ram31}. One can argue that the theory
presented here (in Section \ref{two}) is a formalization and
extension of the betting odds approach. Geometrically, the result
says that there is a hyper-plane in the space of contracts that
separates accept from reject. We generalize this statement, by using
the Hahn-Banach Theorem, to the countable case where the set of
hyper-planes (the dual space) depends on the space of contract. The
answers for different cases can then be found in the Banach space
theory literature. This provides a new approach to understanding
issues like finite vs.\ countable additivity. We take advantage of
this to formulate rational agents that can deal successfully with
countable (possibly universal as in all computable environments)
classes of environments.

Our presentation begins in Section \ref{two} by first looking at a
fundamental case where one has to accept or reject certain contracts
defining positive and negative rewards that depend on the outcome of
an event with finitely many possibilities. To draw the conclusion
that there are implicit unique probabilistic beliefs, it is
important that the decision maker has an opinion (acceptable,
rejectable or both) on every possible contract. This is what we mean
when we say \emph{complete decision maker}.

In a more general setting, we consider sequential decision
making where given any contract on the sequence of observations
and actions, the decision maker must be able to choose a policy
(i.e.\ an action tree). Note that the actions may affect the
environment. A contract on such a sequence can e.g.\ be viewed
as describing a reward structure for a task. An example of a
task is a cleaning robot that gets positive rewards for
collecting dust and negative for falling down the stairs. A
prerequisite for being able to continue to collect dust can be
to recharge the battery before running out. A specialized
decision maker that deals only with one contract/task does not
always need to have implicit probabilities, it can suffice with
qualitative beliefs to take reasonable decisions. A qualitative
belief can be that one pizza delivery company (e.g.\ Pizza Hut
vs Dominos) is more likely to arrive on time than the other. If
one believes the pizzas are equally good and the price is the
same, we will chose the company we believe is more often
delivering on time. Considering all contracts (reward
structures) on the actions and events, leads to a situation
where having a way of making rational (coherent) decisions,
implies that the decision maker has implicit probabilistic
beliefs. We say that the probabilities are implicit because the
decision maker, which might e.g.\ be a human, a dog, a computer
or just a set of rules, might have a non-probabilistic
description of how the decisions are made.

In Section \ref{three}, we investigate extensions to the case with
countably many possible outcomes and the interesting issue of
countable versus finite additivity. Savage's axioms are known to
only lead to finite additivity while \cite{Arrow70} showed that
adding a monotone continuity assumption guarantees countable
additivity. We find that in our setting, it depends on the space of
contracts in an interesting way. In Section \ref{four}, we discuss a
setting where we have a class of environments.

\section{Rational Decisions for Accepting or Rejecting Contracts}\label{two}

We consider a setting where we observe a
symbol (letter) from a finite alphabet and we are offered a form of
bet we call a contract that we can accept or not.
\begin{definition}[Passive Environment, Event]
A passive environment is a sequence of symbols (letters) $j_t$,
called events, being presented one at a time. At time $t$ the
symbols $j_1,...,j_t$ are available. We can equivalently say that a
passive environment is a function $\nu$ from finite strings to
$\{0,1\}$ where $\nu(j_1,...,j_t)=1$ if and only if the environment
begins with $j_1,...,j_t$.
\end{definition}

\begin{definition}[Contract]
Suppose that we have a passive environment with symbols from an
alphabet with $m$ elements. A contract for an event is an element
$x=(x_1,...,x_m)$ in $\mathbb{R}^m$ and $x_j$ is the reward received
if the event is the $j$:th symbol, under the assumption that the
contract is accepted (see next definition).
\end{definition}

\begin{definition}[Decision Maker, Decision]
A decision maker (for some unknown environment) is a set
$Z\subset\mathbb{R}^m$ which defines exactly the contracts that are
acceptable. In other words, a decision maker is a function from
$\mathbb{R}^m$ to $\{$accepted, rejected, either$\}$. The function
value is called the decision.
\end{definition}

 If $x\in Z$ and $\lambda\geq 0$ then we want $\lambda
x\in Z$ since it is simply a multiple of the same contract. We also
want the sum of two acceptable contracts to be acceptable. If we
cannot lose money we are prepared to accept the contract. If we are
guaranteed to win money we are not prepared to reject it. We
summarize these properties in the definition below of a rational
decision maker.
\begin{definition}[Rationality I]
We say that the decision maker ($Z\subset\mathbb{R}^m$) is rational
if
\begin{enumerate}
\item Every contract $x\in\mathbb{R}^m$ is either acceptable or rejectable or
both;
\item $x$ is acceptable if and only if $-x$ is rejectable;
\item $x,y\in Z$, $\lambda,\gamma \geq 0$ then $\lambda x+\gamma y\in Z$;
\item If $x_k\geq 0\ \forall k$ then $x=(x_1,...,x_m)\in Z$ while
if $x_k<0\ \forall k$ then $x\notin Z$.
\end{enumerate}
\end{definition}

If we want to compare these axioms to rationality axioms for a
preference relation on contracts we will say that $x$ is better or
equal (as in equally good) than $y$ if $x-y$ is acceptable while it is
worse or equal if $x-y$ is rejectable. The first axiom is
completeness. The second says that if $x$ is better or equal than $y$
then $y$ is worse or equal to $x$. The third implies transitivity
since $(x-y)+(y-z)=(x-z)$. The fourth says that if $x$ has a better
(or equal) reward than $y$ for any event, then $x$ is better (or
equal) than $y$.

\subsection{Probabilities and Expectations}

\begin{theorem}[Existence of Probabilities]\label{thm:P}
Given a rational decision maker, there are numbers $p_i\geq 0$ that
satisfy
\begin{equation}\label{impl}
\{x\ |\ \sum x_ip_i> 0\}\subset Z\subseteq\{x\ |\ \sum x_ip_i\geq 0\}.
\end{equation}
Assuming $\sum_i p_i=1$ makes the numbers unique and we will use the
notation $Pr(i)=p_i$.
\end{theorem}
\begin{proof}
See the proof of the more general Theorem \ref{thm:LF}. It tells us
that the closure $\bar{Z}$ of $Z$ is a closed half space and can be
written as $\{x\ |\ \sum x_ip_i\geq 0\}$ for some vector $p=(p_i)$
(since every linear functional on $\mathbb{R}^m$ is of the form
$f(x)=\sum x_ip_i$) and not every $p_i$ is $0$. The fourth property
tells us that $p_i\geq 0\ \forall i$.
\end{proof}

\begin{definition}[Expectation]
We will refer to the function $g(x)=\sum p_ix_i$ from \eqref{impl}
as the decision makers expectation. In this terminology, a rational
decision maker has an expectation function and accepts a contract
$x$ if $g(x)>0$ and reject it if $g(x)<0$.
\end{definition}

\begin{remark}
Suppose that we have a contract $x=(x_i)$ where $x_i=1$ for all $i$.
If we want $g(x)=1$, we need $\sum p_i=1$.
\end{remark}
We will write $E(x)$ instead of $g(x)$ (assuming $\sum p_i=1$)
from now on and call it the expected value or expectation of
$x$.

\subsection{Multiple Events}

Suppose that the contract is such that we can view the symbol to be
drawn as consisting of two (or several) symbols from smaller
alphabets. That is we can write a drawn symbol as $(i,j)$ where all
the possibilities can be found through $1\leq i\leq m$, $1\leq j\leq
n$. In this way of writing, a contract is defined by real numbers
$x_{i,j}$. Theorem \ref{thm:P} tells us that for a rational decision
maker there exists unique $r_{i,j}\geq 0$ such that $\sum_{i,j}
r_{i,j}=1$ and an expectation function $g(x)=\sum r_{i,j}x_{i,j}$
such that contracts are accepted if $g(x)>0$ and rejected if
$g(x)<0$.

\subsection{Marginals}\label{marginal}

Suppose that we can take rational decisions on bets for a pair of
horse races, while the person that offers us bets only cares about
the first race. Then we are still equipped to respond since the bets
that only depend on the first race is a subset of all bets on the
pair of races.
\begin{definition}[Marginals]
Suppose that we have a rational decision maker ($Z$) for contracts
on the events $(i,j)$. Then we say that the marginal decision maker
for the first symbol ($Z_1$)  is the restriction of the decision
maker $Z$ to the contracts $x_{i,j}$ that only depend on $i$, i.e.\
$x_{i,j}=x_i$. In other words given a contract $y=(y_i)$ on the
first event, we extend that contract to a contract on $(i,j)$ by
letting $y_{i,j}=y_i$ and then the original decision maker can
decide.
\end{definition}

Suppose that $x_{i,j}=x_i$. Then the expectation $\sum
r_{i,j}x_{i,j}$ can be rewritten as $\sum p_i x_i$ where $p_i=\sum_j
r_{i,j}$. We write that
$$
  Pr(i)=\sum_j Pr(i,j).
$$
These are the marginal probabilities for the first variable
that describe the marginal decision maker for that variable.
Naturally we can also define a marginal for the second variable
(considering contracts $x_{i,j}=x_j$) by letting $q_j=\sum_i
r_{i,j}$ and $Pr(j)=\sum_i Pr(i,j)$. The marginals define sets
$Z_1\subset\mathbb{R}^m$ and $Z_2\subset\mathbb{R}^n$ of
acceptable contracts on the first and second variables
separately.

\subsection{Conditioning}\label{conditioning}

Again suppose that we are taking decisions on bets for a pair of
horse races, but this time suppose that the first race is already
over and we know the result. We are still equipped to respond to
bets on the second race by extending the bet to a bet on both where
there is no reward for (pairs of) events that are inconsistent with
what we know.
\begin{definition}[Conditioning]\label{def:cond}
Suppose that we have a rational decision maker ($Z$) for contracts
on the events $(i,j)$. We define the conditional decision maker
$Z_{j=j_0}$ for $i$ given $j=j_0$ by restricting the original
decision maker $Z$ to contracts $x_{i,j}$ which are such that
$x_{i,j}=0$ if $j\neq j_0$. In other words if we start with a
contract $y=(y_i)$ on $i$ we extend it to a contract on $(i,j)$ by
letting $y_{i,j_0}=y_i$ and $y_{i,j}=0$ if $j\neq j_0$. Then the
original decision maker can make a decision for that contract.
\end{definition}

Suppose that $x_{i,j}=0$ if $j\neq j_0$. The unconditional
expectation of this contract is $\sum_{i,j} r_{i,j}x_{i,j}$ as usual
which equals $\sum_i r_{i,j_0}x_{i,j_0}$. This leads to the same
decisions (i.e.\ the same $Z$) as using $\sum_i
\frac{r_{i,j_0}}{\sum_k r_{k,j_0}} x_{i,j_0}$ which is of the form
in Theorem \ref{thm:P}. We write that
\begin{equation}
Pr(i|j_0)=\frac{Pr(i,j_0)}{\sum_k
Pr(k,j_0)}=\frac{Pr(i,j_0)}{Pr(j_0)}.
\end{equation}
From this it follows that
\begin{equation}\label{Bayes}
Pr(i_0)Pr(j_0|i_0)=Pr(j_0)Pr(i_0|j_0)
\end{equation}
which is one way of writing Bayes rule.

\subsection{Learning}

In the previous section we defined conditioning which lead us
to a definition of what it means to learn. Given that we have
probabilities for events that are sequences of a certain number
of symbols and we have observed one or several of them, we use
conditioning to determine what our belief regarding the
remaining symbols should be.

\begin{definition}[Learning]
Given a rational decision maker, defined by $p_{i_1,...,i_T}$ for
the events $(i_t)_{t=1}^T$ and the first $t-1$ symbols
$i_1,...,i_{t-1}$, we define the informed rational decision maker
for $i_{t}$ by conditioning on the past $i_1,...,i_{t-1}$
and marginalize over the future $i_{t+1},...,i_T$. Formally,
$$
  P_{i_{t}}^{\text{informed}}(i)=Pr(i|i_1,...,i_{t})=\frac{\sum_{j_{t+1},...,j_T}
  p_{i_1,...,i_{t},j_{t+1},...,j_T}}{\sum_{j_{t},...,j_T}
  p_{i_1,...,i_{t-1},j_{t},...,j_T}}.
$$
\end{definition}

\subsection{Choosing between Contracts}

\begin{definition}[Choosing contract]
We say that to rationally prefer contract $x$ over $y$ is
(equivalent) to rationally consider $x-y$ to be acceptable.
\end{definition}
As before we assume that we have a decision maker that takes
rational decisions on accepting or rejecting contracts $x$ that are
based on an event that will be observed. Hence there exist implicit
probabilities that represent all choices and an expectation
function. Suppose that an agent has to choose between action $a_1$
that leads to receiving reward $x_i$ if $i$ is drawn and action
$a_2$ that leads to receiving $y_i$ in the case of seeing $i$. Let
$z_i=x_i-y_i$. We can now go back to choosing between accepting and
rejecting a contract by saying that choosing (preferring) $a_1$ over
$a_2$ means accepting the contract $z$. In other words if
$E(x)>E(y)$ choose $a_1$ and if $E(x)<E(y)$ choose $a_2$.

\begin{remark}
We note that if we postulate that choosing between contract $x$ and
the zero contract is the same as choosing between accepting or
rejecting $x$, then being able to choose between contracts implies
the ability to choose between accepting and rejecting one contract.
We, therefore, can say that the ability to choose between a pair of
contracts is equivalent to the ability to choose to accept or reject
a single contract.
\end{remark}

We can also choose between several contracts. Suppose that action $a_k$
gives us the contract $x^k=(x^k_i)_{i=1}^m$. If $E(x^j)> E(x^k)\
\forall k\neq j$ then we strictly prefer $a_j$ over all other
actions. In other words a contract $x^j-x^k$ would for all $k$ be
accepted and not rejected by a rational decision maker.
\begin{remark}
If we have a rational decision maker for accepting or rejecting
contracts, then there are implicitly probabilities $p_i$ for symbol
$i$ that characterize the decisions. A rational choice between
actions $a_k$ leading to contracts $x^k$ is taken by choosing action
\begin{equation}
a^*=\arg\max_{k} \sum_i p_ix^k_i.
\end{equation}
\end{remark}

\subsection{Choosing between Environments}

In this section, we assume that the event that the contracts are
concerned with might be affected by the choice of action.

\begin{definition}[Reactive environment]
An environment is a tree with symbols $j_t$ (percepts) on the nodes
and actions $a_t$ on the edges. We provide the environment with an
action $a_t$ at each time $t$ and it presents the symbol $j_t$ at
the node we arrive at by following the edge chosen by the action. We
can also equivalently say that a reactive environment $\nu$ is a
function from strings $a_1j_1,...,a_tj_t$ to $\{0,1\}$ which equals
$1$ if and only if $\nu$ would produce $j_1,...,j_t$ given the
actions $a_1,...,a_t$.
\end{definition}

We will define the concept of a decision maker for the case where one
decision will be taken in a situation where not only the contract,
but also the outcome can depend on the choice. We do this by
defining the choice as being between two different environments.

\begin{definition}[Active decision maker]\label{adm}
Consider a choice between having contract $x$ for passive
environment $env_1$ or contract $y$ for passive environment $env_2$.
A decision maker is a set $Z\subset
\mathbb{R}^{m_1}\times\mathbb{R}^{m_2}$ which defines exactly the
pairs $(x,y)$ for which we choose $env_1$ with $x$ over $env_2$ with
$y$.
\end{definition}

\begin{definition}[Rational active choice]\label{rac}
To choose between action $a_1$ with contract $x$ and $a_2$ with
contract $y$ in a situation where the action may affect the event,
we consider two separate environments, namely the environments that
result from the two different actions. We would then have a
situation where we will have one observation from each environment.
Preferring $a_1$ with $x$ to $a_2$ with $y$ is (equivalent) to
consider $x-y$ to be an acceptable contract for the pair of events.
\end{definition}

\begin{remark}\label{reformulate}
Definition \ref{rac} means that $a_1$ with $x$ is preferred over
$a_2$ with $y$ if $a_1$ with $x-y$ is preferred over $a_2$ with the
zero contract.
\end{remark}

\begin{proposition}[Probabilities for reactive setting]\label{actprob}
Suppose that we have a reactive environment and a rational active
decision maker that will make one choice between action $a_1$ and
$a_2$ as described in Definitions \ref{adm} and \ref{rac}, then
there exist $p_i\geq 0$ and $q_i\geq 0$ such that action $a_1$ with
contract $x$ is preferred over action $a_2$ with contract $y$ if
$\sum p_ix_i>\sum q_i y_i$ and the reverse if $\sum p_ix_i<\sum
q_iy_i$. This means that the decision maker acts according to
probabilities $Pr(\cdot|a_1)$ and $Pr(\cdot|a_2)$.
\end{proposition}

\begin{proof}
Let $\tilde{Z}$ be all contracts that when combined with action
$a_1$ is preferred over $a_2$ with the zero contract. Theorem
\ref{impl} guarantees the existence of $p_i$ such that $\sum
p_ix_i>0$ implies that $x\in \tilde{Z}$ and $\sum p_ix_i<0$ implies
that $x\notin \tilde{Z}$. The same way we find $q_i$ that describe
when we prefer $a_2$ with $y$ to $a_1$ with the zero contract. That
these probabilities ($p_i$ and $q_i$) explain the full decision
maker as stated in the proposition now follows directly from
Definition \ref{rac} understood as in Remark \ref{reformulate}.
\end{proof}

Suppose that we are going to make a sequence of $T<\infty$ decisions
where at every point of time we will have a finite number of actions
to chose between. We will consider contracts, which can pay out some
reward at each time step and that can depend on everything (actions
chosen and symbols observed) that has happened up until this time
and we want to maximize the accumulated reward at time $T$.

We can view the choice as just making one choice, namely choosing an
action tree. We will sometimes call an action tree a policy.

\begin{definition}[Action tree]
An action tree is a function from histories of symbols
$j_1,...,j_{t}$ and decisions $a_1,...,a_{t-1}$ to new decisions,
given that the decisions were made according to the function.
Formally,
$$
  f(a_1,j_1,...,a_{t-1},j_{t-1})=a_{t}.
$$
\end{definition}

 An action tree will assign exactly one action for any
of the circumstances that one can end up in. That is, given the
history up to any time $t<T$ of actions and events, we have a chosen
action. We can, therefore, choose an action tree at time $0$ and
receive a total accumulated reward at time $T$. This brings us back
to the situation of one event and one rational choice.

\begin{definition}[Sequential decisions]\label{def:sd}
Given a rational decision maker for the events $(j_t)_{t=1}^T$ and
the first $t-1$ symbols $j_1,...,j_{t-1}$ and decisions
$a_1,...,a_{t-1}$, we define the informed rational decision maker at
time $t$ by conditioning on the past $a_1,j_1...,a_{t-1},j_{t-1}$.
\end{definition}

\begin{proposition}[Beliefs for sequential decisions]
Suppose that we have a reactive environment and a rational decision
maker that will take $T<\infty$ decisions. Furthermore, suppose that
the decisions $0\leq t<T$ have been taken and resulted in
history $a_1,j_1...,a_{t-1},j_{t-1}$. Then the decision makers
preferences at this time can be explained (through expected utility
maximization) by probabilities
$$
  Pr(j_t,...,j_T|a_1,j_1...,a_{t-1},j_{t-1},a_t,a_{t+1}...,a_T).
$$
\end{proposition}

\begin{proof}
Definition \ref{def:sd} and Proposition \ref{actprob} immediately
lead us to the conclusion that given a past up to a point $t-1$ and
a policy for the time $t$ to $T$ we have probabilistic beliefs over
the possible future sequences from time $t$ to $T$ and the choice is
categorized by maximizing expected accumulated reward at time $T$.
\end{proof}

\section{Countable Sets of Events}\label{three}

Instead of a finite set of possible outcomes, we will in this section
assume a countable set. We suppose that the set of contracts
is a vector space of sequences $x_k, k=0,1,2,...$ where we use
pointwise addition and multiplication with scalar. We will define a
space by choosing a norm and let the space consist of the sequences
that have finite norm as is common in Banach space theory. If the
norm makes the space complete it is called a Banach sequence space
\cite{Diestel84}. Interesting examples are $\ell^\infty$ of bounded
sequences with the maximum norm $\|(\alpha_k)\|_\infty=\max
|\alpha_k|$, $c_0$ of sequence that converges to $0$ equipped with
the same maximum norm and $\ell^p$ which for $1\leq p<\infty$ is
defined by the norm
$$
  \|(\alpha_k)\|_p=(\sum |\alpha_k|^p)^{1/p}.
$$
For all of these spaces we can consider weighted versions ($w_k>0$)
where
$$
  \|(\alpha_k)\|_{p,w_k}=\|(\alpha_kw_k)\|_p.
$$
This means that $\alpha\in\ell^p(w)$ iff $(\alpha_kw_k)\in\ell^p$,
e.g.\ $\alpha\in\ell^\infty(w)$ iff $\sup_k |\alpha_kw_k|<\infty$.
Given a Banach (sequence) space $X$ we use $X'$ to denote the dual
space that consists of all continuous linear functionals
$f:X\to\mathbb{R}$. It is well known that a linear functional on a
Banach space is continuous if and only if it is bounded, i.e.\ that
there is $C<\infty$ such that $\frac{|f(x)|}{\|x\|}\leq C\ \forall
x\in X$. Equipping $X'$ with the norm $\|f\|=\sup
\frac{|f(x)|}{\|x\|}$ makes it into a Banach space. Some examples
are $(\ell^1)'=\ell^\infty$, $c_0'=\ell^1$ and for $1<p<\infty$ we
have that $(\ell^p)'=\ell^q$ where $1/p+1/q=1$. These
identifications are all based on formulas of the form
$$
  f(x)=\sum x_i p_i
$$
where the dual space is the space that $(p_i)$ must lie in to
make the functional both well defined and bounded. It is clear
that $\ell^1\subset(\ell^\infty)'$ but $(\ell^\infty)'$ also
contains ``stranger" objects.

The existence of these other objects can be deduced from the
Hahn-Banach theorem (see e.g.\ \cite{Krey89} or \cite{Nar97}) that
says that if we have a linear function defined on a subspace $Y\in
X$ and if it is bounded on $Y$ then there is an extension to a
bounded linear functional on $X$. If $Y$ is dense in $X$ the
extension is unique but in general it is not. One can use this
Theorem by first looking at the subspace of all sequences in
$\ell^\infty$ that converge and let $f(\alpha)=\lim_{k\to\infty}
\alpha_k$. The Hahn-Banach theorem guarantees the existence of
extensions to bounded linear functionals that are defined on all of
$\ell^\infty$. These are called Banach limits. The space
$(\ell^\infty)'$ can be identified with the so called ba space of
bounded and finitely additive measures with the variation norm
$\|\nu\|=|\nu|(A)$ where $A$ is the underlying set. Note that
$\ell^1$ can be identified with the smaller space of countably
additive bounded measures with the same norm. The Hahn-Banach
Theorem has several equivalent forms. One of these identifies the
hyper-planes with the bounded linear functionals \cite{Nar97}.

\begin{definition}[Rationality II]
Given a Banach sequence space $X$ of contracts, we say that the
decision maker (subset $Z$ of $X$ defining acceptable contracts) is
rational if
\begin{enumerate}
\item Every contract $x\in X$ is either acceptable or rejectable or
both;
\item $x$ is acceptable if and only if $-x$ is rejectable;
\item $x,y\in Z$, $\lambda,\gamma \geq 0$ then $\lambda x+\gamma y\in Z$;
\item If $x_k\geq 0\ \forall k$ then $x=(x_k)$ is acceptable while
if $x_k>0\ \forall k$ then $x$ is not rejectable.
\end{enumerate}
\end{definition}

\begin{theorem}[Linear separation]\label{thm:LF}
Suppose that we have a space of contracts $X$ that is a Banach
sequence space. Given a rational decision maker there is a positive
continuous linear functional $f:X\to\mathbb{R}$ such that
\begin{equation}
\{x\ |\ f(x)> 0\}\subset Z\subseteq\{x\ |\ f(x)\geq 0\}.
\end{equation}
\end{theorem}
\begin{proof} The third property tells us that $Z$ and $-Z$ are convex
cones. The second and fourth property tells us that $Z\neq
\mathbb{R}^m$. Suppose that there is a point $x$ that lies in both
the interior of $Z$ and of $-Z$. Then the same is true for $-x$
according to the second property and for the origin. That a ball
around the origin lies in $Z$  means that $Z=\mathbb{R}^m$ which is
not true. Thus the interiors of $Z$ and $-Z$ are disjoint open
convex sets and can, therefore, be separated by a hyperplane
(according to the Hahn-Banach theorem) which goes through the origin
(since according to the second and fourth property the origin is
both acceptable and rejectable). The first two properties tell us
that $Z\cup -Z=\mathbb{R}^m$. Given a separating hyperplane (between
the interiors of $Z$ and $-Z$), $Z$ must contain everything on one
side. This means that $Z$ is a half space whose boundary is a
hyperplane that goes through the origin and the closure $\bar{Z}$ of
$Z$ is a closed half space and can be written as $\{x\ |\ f(x)\geq
0\}$ for some $f\in X'$. The fourth property tells us that $f$ is
positive.
\end{proof}

\begin{corollary}[Additivity]\label{cor:add}
1. If $X=c_0$ then a rational decision maker is described by a
countably additive (probability) measure.
\\2. If $X=\ell^\infty$ then a rational decision maker is described by a
finitely additive (probability) measure.
\end{corollary}

It seems from Corollary \ref{cor:add} that we pay the price of
losing countable additivity for expanding the space of contracts
from $c_0$ to $\ell^\infty$ but we can expand the space even more by
looking at $c_0(w)$ where $w_k\to 0$ which contains $\ell^\infty$
and $X'$ is then $\ell^1((1/w_k))$. This means that we get countable
additivity back but we instead have a restriction on how fast the
probabilities $p_k$ must tend to $0$. Note that a bounded linear
functional on $c_0$ can always be extended to a bounded linear
functional on $\ell^\infty$ by the formula $f(x)=\sum p_ix_i$ but
that is not the unique extension. Note also that every bounded
linear functional on $\ell^\infty$ can be restricted to $c_0$ and
there be represented as $f(x)=\sum p_ix_i$. Therefore, a rational
decision maker on $\ell^\infty$ contracts has probabilistic beliefs
(unless $p_i=0\ \forall i$), though it might also take asymptotic
behavior of a contract into account. For example (and here $p_i=0\
\forall i$), the decision maker that makes decisions based on
asymptotic averages $\lim_{n\to\infty} \frac{1}{n}\sum_{i=1}^n x_i$
when they exist. That strategy can be extended to all of
$\ell^\infty$ (a Banach limit). The following proposition will help
us decide which decision maker on $\ell^\infty$ is described with
countably additive probabilities.

\begin{proposition}\label{p:approx}
Suppose that $f\in(\ell^\infty)'$. For any $x\in\ell^\infty$,
let $x^j_i=x_i$ if $i\leq j$ and $x^j_i=0$ otherwise. If for
any $x$,
$$
  \lim_{j\to\infty} f(x^j)=f(x),
$$
then $f$ can be written as $f(x)=\sum p_ix_i$ where $p_i\geq 0$
and $\sum_{i=1}^\infty p_i<\infty$.
\end{proposition}
\begin{proof}
The restriction of $f$ to $c_0$ gives us numbers $p_i\geq 0$ such
that $\sum_{i=1}^\infty p_i<\infty$ and $f(x)=\sum p_ix_i$ for $x\in
c_0$. This means that $f(x^j)=\sum_{i=1}^j p_ix_i$ for any
$x\in\ell^\infty$ and $j<\infty$. Thus $\lim_{j\to\infty}
f(x^j)=\sum_{i=1}^\infty p_ix_i$.
\end{proof}

\begin{definition}[Monotone decisions]
We define the concept of a \em{monotone} decision maker in the
following way. Suppose that for every $x\in\ell^\infty$ there is
$N<\infty$ such that the decision is the same for all $x^j,\ j\geq
N$ (See Proposition \ref{p:approx} for definition) as for $x$. Then
we say that the decision maker is monotone.
\end{definition}

\begin{example}
Let $f\in\ell^\infty$ be such that if $\lim \alpha_k\to L$ then
$f(\alpha)=L$ (i.e.\ $f$ is a Banach limit). Furthermore define a
rational decision maker by letting the set of acceptable contracts
be $Z=\{x\ |\ f(x)\geq 0\}$. Then $f(x^j)=0$ (where we use notation
from Proposition \ref{p:approx}) for all $j<\infty$ and regardless
of which $x$ we define $x^j$ from. Therefore, all sequences that are
eventually zero are acceptable contracts. This means that this
decision maker is not monotone since there are contracts that are
not acceptable.
\end{example}

\begin{theorem}[Monotone rationality]\label{thm:mono}
Given a monotone rational decision maker for $\ell^\infty$
contracts, there are $p_i\geq 0$ such that $\sum p_i<\infty$ and
\begin{equation}
\{x\ |\ \sum x_ip_i>0\}\subset Z\subseteq \{x\ |\ \sum x_ip_i\geq
0\}.
\end{equation}
\end{theorem}
\begin{proof}
According to Theorem \ref{thm:LF} there is $f\in(\ell^\infty)'$ such
that (the closure of $Z$) $\bar{Z}=\{x|\ f(x)\geq 0\}$ . Let
$p_i\geq 0$ be such that $\sum p_i<\infty$ and such that $f(x)=\sum
x_ip_i$ for $x\in c_0$. Remember that $x^j$ (notation as in
Proposition \ref{p:approx}) is always in $c_0$. Suppose that there
is $x$ such that $x$ is accepted but $\sum x_ip_i<0$. This violate
monotonicity since there exist $N<\infty$ such that $\sum_{i=1}^n
x_ip_i<0$ for all $n\geq N$ and, therefore, $x^j$ is not accepted
for $j\geq N$ but $x$ is accepted. We conclude that if $x$ is
accepted then $\sum p_ix_i \geq 0$ and if $\sum p_ix_i>0$ then $x$
is accepted.
\end{proof}

\section{Rational Agents for Classes of Environments}\label{four}

We will here study agents that are designed to deal with a large
range of situations. Given a class of environments we want to define
agents that can learn to act well when placed in any of them,
assuming it is at all possible.

\begin{definition}[Universality for a class]
We say that a decision maker is universal for a class of
environments $\mathcal{M}$ if for any outcome sequence
$a_1j_1a_2j_2...$ that given the actions would be produced by some
environment in the class, there is $c>0$ (depending on the sequence)
such that the decision maker has probabilities that satisfy
$$
  Pr(j_1,...,j_t|a_1,...,a_t)\geq c\ \forall t.
$$
This is obviously true if the decision maker's probabilistic
beliefs are a convex combination $\sum_{\nu\in \mathcal{M}}
w_\nu \nu$, $w_\nu>0$ and $\sum_\nu w_\nu=1$.
\end{definition}

We will next discuss how to define some large classes of
environments and agents that can succeed for them. We assume that
the total accumulated reward from the environment will be finite
regardless of our actions since we want any policy to have finite
utility. Furthermore, we assume that rewards are positive and that
it is possible to achieve strictly positive rewards in any
environment. We would like the agent to perform well regardless of
which environment from the chosen class it is placed in.

For any possible policy (action tree) $\pi$ and environment $\nu$,
there is a total reward $V_\nu^\pi$ that following $\pi$ in $\nu$
would result in. This means that for any $\pi$ there is a contract
sequence $(V_\nu^\pi)_\nu$, assuming we have enumerated our set of
environments. Let
$$
  V^*_\nu=\max_\pi V_\nu^\pi.
$$
We know that $V^*_\nu>0$ for all $\nu$. Every contract sequence
$(V_{\nu}^\pi)_\nu$ lies in $X=\ell^\infty((1/V^*_\nu))$ and
$\|(V_\nu^\pi)\|_X\leq 1$. The rational decision makers are the
positive, continuous linear functionals on $X$. $X'$ contains the
space $\ell^1(V^*_\nu)$. In other words if $w_\nu \geq 0$ and $\sum
w_\nu V_\nu^*<\infty$ then the sequence $(w_\nu)$ defines a rational
decision maker for the contract space $X$. These are exactly the
monotone rational decision makers. Letting (which is the AIXI agent
from \cite{Hutter04})
\begin{equation}\label{eq:AIXI}
\pi^*\in\arg\max_\pi \sum_\nu w_\nu V_\nu^\pi
\end{equation}
we have a choice with the property that for any other $\pi$ with
$$
  \sum_\nu w_\nu V_\nu^\pi< \sum_\nu w_\nu V_\nu^{\pi^*}.
$$
Hence the contract $(V_\nu^{\pi^*}-V_\nu^\pi)$ is not rejectable. In
other words $\pi^*$ is strictly preferable to $\pi$. By letting
$p_\nu=w_\nu V^*_\nu$, we can rewrite \eqref{eq:AIXI} as
\begin{equation}\label{eq:AIXI2}
\pi^*\in\arg\max_\pi \sum_\nu p_\nu \frac{V_\nu^\pi}{V^*_\nu}.
\end{equation}
If one further restricts the class of environments by assuming
$V^*_\nu\leq 1$ for all $\nu$ then for every $\pi$,
$(V^\pi_\nu)\in\ell^\infty$. Therefore, by Theorem \ref{thm:mono}
the monotone rational agents for this setting can be formulated as
in \eqref{eq:AIXI} with $(w_\nu)\in\ell_1$, i.e.\ $\sum_\nu
w_\nu<\infty$. However, since $(p_\nu)\in\ell_1$, a formulation of
the form of \eqref{eq:AIXI2} is also possible. Normalizing $p$ and
$w$ individually to probabilities makes \eqref{eq:AIXI} into a
maximum expected utility criterion and \eqref{eq:AIXI2} into maximum
relative utility. As long as our $w$ and $p$ relate the way they do
it is still the same decisions. If we would base both expectations
on the same probabilistic beliefs it would be different criteria.
When we have an upper bound $V^*_\nu<b<\infty\ \forall\nu$ we can
always translate expected utility to expected relative utility in
this way, while we need a lower bound $0<a<V^*_\nu$ to rewrite an
expected relative utility as an expected utility. Note, the
different criteria will start to deviate from each other after
updating the probabilistic beliefs.

\subsection{Asymptotic Optimality}

Denote a chosen countable class of environments by $\mathcal{M}$.
Let $V_{\nu,k}^\pi$ be the rewards achieved after time $k$ using
policy $\pi$ in environment $\nu$. We suppress the dependence on the
history so far. Let
$$
  W_{\nu,k}^\pi=\frac{V_{\nu,k}^\pi}{V_{\nu,k}^*}
$$
denote the skill (relative reward) of $\pi$ in environment $\nu$
from time $k$. The maximum possible skill is $1$. We would like to
have a policy $\pi$ such that
$$
  \lim_{k\to\infty} W_{\nu,k}^\pi=1\ \forall \nu\in\mathcal{M}.
$$
This would mean that the agent asymptotically achieve maximum skill
when placed in any environment from $\mathcal{M}$. Let
$I(h_k,\nu)=1$ if $\nu$ is consistent with history $h_k$ and
$I(h_k,\nu)=0$ otherwise. Furthermore, let
$$
  p_{\nu,k}=\frac{p_{\nu,0}}{\sum_{\mu\in\mathcal{M}}p_{\mu,0}I(h_k,\mu)}
$$
be the agent's weight for environment $\nu$ at time
$k$ and let $\pi^p$ be a policy that at time $k$ acts according to a
policy in
\begin{equation}\label{eq:AIXI2p}
\arg\max_\pi \sum_\nu p_{\nu,k} \frac{V_{\nu,k}^\pi}{V^*_{\nu,k}}.
\end{equation}

In the following theorem, we prove that for every environment
$\nu\in \mathcal{M}$, the policy $\pi^{p}$ will asymptotically
achieve perfect relative rewards. We have to assume that there
exists a sequence of policies $\pi_k>0$ with this property (as for
the similar Theorem 5.34 in \cite{Hutter04} which dealt with
discounted values). The convergence in $W$-values is the relevant
sense of optimality for our setting, since the $V$-values converge
to zero for any policy.

\begin{theorem}[Asymptotic optimality]\label{thmOpt}
Suppose that we have a decision maker that is universal (i.e.
$p_\nu>0\ \forall\nu$) with respect to the countable class
$\mathcal{M}$ of environments (which can be stochastic) and that
there exists policies $\pi_k$ such that for all $\nu$,
$W_k^{\pi_k,\nu}\to 1$ if $\nu$ is the actual environment (or the
sequence is consistent with $\nu$). This implies that
$W_k^{\pi^{p},\mu}\to 1$ where $\mu$ is the actual environment.

\end{theorem}

The proof technique is similar to that of Theorem 5.34 in
\cite{Hutter04}.

\begin{proof}
Let
\begin{equation}
0\leq 1-W^{\pi_k,\nu}_k=:{\Delta}^k_{\nu},\ {\Delta}^k=\sum_{\nu}
p_{\nu,k}{\Delta}^k_{\nu}.
\end{equation}
The assumptions tells us that  $\Delta_\nu^k=W_k^{\pi_k,\nu}- 1\to
0$ for all $\nu$ that are consistent with the sequence
($p_{\nu,k}=0$ if $\nu$ is inconsistent with the history at time
$k$) and since $\Delta^k_\nu\leq 1$ , it follows that
$$
  \Delta^k=\sum_{\nu} p_{\nu,k}\Delta^k_{\nu}\to 0.
$$

Note that $p_{\mu,k} (1-W_k^{\pi^{p},\mu})\leq \sum_{\nu} p_{\nu,k}
(1-W_{\pi^{p},k}^{\nu})\leq \sum_{\nu} p_{\nu,k}
(1-W^k_{\pi_k,\nu})=\sum p_{\nu,k} \Delta^k_\nu=\Delta^k$. Since we
also know that $p_{\mu,k}\geq p_{\mu,0}>0$ it follows that
$(1-W_k^{\pi^{p},\mu})\to 0$.
\end{proof}

\section{Conclusions}\label{five}

We studied complete rational decision makers including the cases of
actions that may affect the environment and sequential decision
making. We set up simple common sense rationality axioms that imply
that a complete rational decision maker has preferences that can be
characterized as maximizing expected utility. Of particular interest
is the countable case where our results follow from identifying the
Banach space dual of the space of contracts.

\paradot{Acknowledgement}
This work was supported by ARC grant DP0988049.


\newpage
\begin{small}

\end{small}

\end{document}